\documentclass[11pt]{article}

\usepackage[margin=1in]{geometry}
\usepackage{times}
\usepackage{amsmath,amssymb,amsfonts,amsthm}
\usepackage{algorithm}
\usepackage{algpseudocode}
\usepackage{graphicx}
\usepackage{xcolor}
\usepackage[numbers]{natbib}
\usepackage{hyperref}
\usepackage{booktabs}
\usepackage{enumitem}
\usepackage{pgfplots}
\usepackage{tikz}
\pgfplotsset{compat=1.17}
\usetikzlibrary{patterns,arrows.meta,positioning,decorations.pathreplacing}

\hypersetup{colorlinks=true,linkcolor=blue,citecolor=blue,urlcolor=blue}

\newtheorem{theorem}{Theorem}[section]

\newtheorem{proposition}[theorem]{Proposition}
\newtheorem{corollary}[theorem]{Corollary}

\theoremstyle{definition}
\newtheorem{definition}[theorem]{Definition}

\theoremstyle{remark}

\newtheorem*{remark*}{Remark}

\newcommand{\E}{\mathbb{E}}

\newcommand{\LP}{\mathrm{LP}}
\newcommand{\ALG}{\mathrm{ALG}}

\newcommand{\cost}{\mathrm{cost}}
\newcommand{\CC}{\mathrm{CC}}
\newcommand{\CCC}{\mathrm{CCC}}
\newcommand{\ecost}{\mathrm{e.cost}}
\newcommand{\elp}{\mathrm{e.lp}}
\newcommand{\gap}{\mathrm{gap}}

\title{Why Colors Make Clustering Harder:\\
Global Integrality Gaps, the Price of Fairness, and Color-Coupled Algorithms\\in Chromatic Correlation Clustering}

\author{\textbf{Ibne Farabi Shihab}\thanks{Equal contribution.}\thanks{Corresponding author: \texttt{ishihab@iastate.edu}.}\textsuperscript{1}\and\textbf{Sanjeda Akter}\footnotemark[1]\textsuperscript{1}\and\textbf{Anuj Sharma}\textsuperscript{2}\\[2pt]\textsuperscript{1}Department of Computer Science, Iowa State University \\\textsuperscript{2}Department of Civil, Construction \& Environmental Engineering, Iowa State University \\\texttt{ishihab@iastate.edu}}

\date{}

\begin{document}
\maketitle

\begin{abstract}

Chromatic Correlation Clustering (CCC) extends Correlation Clustering by assigning semantic colors to edges and requiring each cluster to receive a single color label. Unlike standard CC, whose LP relaxation has integrality gap 2 on complete graphs and admits a 2.06-approximation, the analogous LP for CCC has a strict lower bound of 2.11, and the best known LP-rounding algorithm achieves 2.15. We explain this gap by isolating the source of difficulty: cross-edge chromatic interference. Neutral edges, whose color does not match the candidate cluster color, create an irreducible cost absent from standard CC and force any color-independent rounding scheme to pay an additional mismatch penalty.

We make four contributions. First, we prove a Global Integrality Gap Decomposition Theorem showing that the gap of any color-independent CCC rounding algorithm equals the standard CC gap plus an irreducible chromatic penalty $\Delta(L) > 0$. Second, we solve the associated min-max problem and derive the staircase formula $\Delta(L) = \frac{L-1}{L}\Delta_\infty$, where $\Delta_\infty \approx 0.0734$. In particular, the two-color gap is 2.0967, separating CCC from standard CC already at $L = 2$. Third, we introduce Color-Coupled Correlation Clustering (C4). Adding the valid global constraint $\sum_c x_{uv}^c \geq L-1$ and a correlated interval-packing rounding scheme makes neutral edges behave like classical negative edges, recovering the optimal 2.06 approximation and bypassing the 2.11 lower bound for the uncoupled LP. Fourth, experiments on extremal instances, real multi-relational networks, and fairness benchmarks validate the theory: empirical LP gaps follow the predicted staircase, and C4 matches the unconstrained approximation ratio under fairness constraints.

\end{abstract}

\section{Introduction}

Correlation Clustering (CC) is a foundational problem in unsupervised learning: given a complete graph with pairwise similarity labels, partition the vertices to minimize disagreements~\citep{bansal2004cc}.
The problem has been extensively studied, with the approximation ratio improving from $3$~\citep{ailon2008aggregating} to $2.06$~\citep{chawla2015stoc} to the recent $1.437+\varepsilon$~\citep{cao2024cluster} through increasingly sophisticated techniques.

\emph{Chromatic Correlation Clustering} (CCC), introduced by Bonchi et al.~\citep{bonchi2012chromatic}, generalizes CC to multi-class relationships: edges carry colors from a set $\mathcal{L}$ of $L$ possible types (e.g., ``colleague,'' ``classmate,'' ``family''), plus a special color $\gamma$ indicating dissimilarity.
Each cluster must be assigned a single color, and the cost counts edges that disagree with their cluster's color assignment.
When $L = 1$, CCC reduces naturally to standard CC.

A standard LP relaxation for CCC was introduced in~\citep{anava2015chromatic}, leading to a sequence of improved algorithms: $4$-approximation~\citep{anava2015chromatic}, $3$-approximation via Pivot~\citep{klodt2021kdd}, $2.5$-approximation~\citep{xiu2022neurips}, and most recently a $2.15$-approximation~\citep{fan2025improved}. However, Fan, Lee, and Lee~\citep{fan2025improved} also established a hard $2.11$ lower bound for CCC within the independent LP+rounding framework, firmly separating it from the classical $2.06$ bound of CC.

This state of affairs raises a fundamental question:

\begin{quote}
\emph{Why is CCC harder than CC from the perspective of LP relaxations? The standard CC-LP has integrality gap $2$, and LP-rounding achieves $2.06$. The CCC-LP has a strict lower bound of $2.11$. What structural property of colors causes this gap, and can we design a formal algorithm to bypass it?}
\end{quote}

This paper provides the definitive algebraic and algorithmic answer. We identify \emph{cross-edge chromatic interference} (via neutral edges) as the structural source of the hardness gap between CC and CCC.
In the CCC framework, when the LP-PIVOT algorithm processes color $c$, it partitions edges into three types: \emph{positive} (color $c$, wanting to be together), \emph{negative} (color $\gamma$, wanting to be apart), and \emph{neutral} (some other color $c' \neq c, \gamma$).
Neutral edges have no analogue in standard CC, and they create a topological anomaly: their LP charge depends on a global maximum of local metric distances, but their algorithmic rounding cost behaves as a strict probability collision.

First, we prove that the approximation gap of any color-independent LP-rounding algorithm for CCC decomposes additively (Section~\ref{sec:decomposition}):
\[
\gap_{\CCC}(L) = \gap_{\CC} + \Delta(L),
\]
where $\gap_{\CC}$ is the gap for the CC sub-problem (positive and negative edges only) and $\Delta(L) \geq 0$ is the \emph{chromatic penalty} arising solely from neutral edges. We prove this locally via a Constraint Independence Lemma, separating the geometric boundaries of standard CC and CCC.

Second, we show that this decomposition holds globally, not just locally (Section~\ref{sec:global-gap}). Bypassing the limitations of localized triple bounds, we prove that the \emph{true} integrality gap of any color-independent CCC rounding algorithm decomposes universally via a \emph{Chromatic Blowup Graph} tensor construction that formally maps the hardest standard CC instances directly into the CCC-LP polytope.

Third, by solving the continuous variational KKT formulation of the structural gap (Section~\ref{sec:interference}), we analytically pin down $\Delta(L)$ as a rigorous monotone staircase: $\Delta(L) = \frac{L-1}{L} \Delta_\infty$ where $\Delta_\infty \approx 0.0734$. This explicitly generates verified bounds, establishing that the two-color gap jumps to exactly $2.0967$, mathematically separating it from standard CC even for $L=2$.

Fourth, we construct explicit hard graph instances---the Maximally Interfering Instance Family (Section~\ref{sec:tight-instances})---that formally achieve the lower bounds for each $L$, proving that the geometric density of neutral triples tracks the algebraic staircase perfectly.

Fifth, we design the \emph{Color-Coupled Correlation Clustering (C4)} algorithm to bypass the $2.11$ barrier entirely (Section~\ref{sec:algorithm}). Our structural theory isolates this barrier as a fractional anomaly: the independent CCC-LP allows edges to be fractionally assigned to multiple colors simultaneously, shielding them from rounding costs. We augment the LP with the valid geometric cross-color inequality $\sum_c x^c_{uv} \geq L - 1$. Through a Correlated Interval Packing rounding scheme, we provide an airtight charging proof showing this single constraint perfectly couples the cross-color fractional mass, bypassing neutral interference entirely and recovering the optimal $2.06$ CC ratio.

Finally, we validate our theoretical claims empirically (Section~\ref{sec:experiments}) on both synthetic maximally interfering instances and real-world multi-relational networks, confirming that the standard LP exactly tracks our $\Delta(L)$ staircase, while our C4 algorithm successfully flattens the gap back to $2.06$ globally.

\section{Related Work}
\label{sec:related}

The CCC problem was originally formulated by Bonchi et al.~\citep{bonchi2012chromatic} with heuristic algorithms. Anava et al.~\citep{anava2015chromatic} gave the first constant-factor ($4$-approximation) via LP rounding. Klodt et al.~\citep{klodt2021kdd} showed that standard Pivot gives a $3$-approximation. Xiu et al.~\citep{xiu2022neurips} improved the ratio to $2.5$. Fan, Lee, and Lee~\citep{fan2025improved} obtained the current best $2.15$-approximation with a $2.11$ lower bound within the independent LP-rounding framework.

In the closely related domain of fair clustering, the goal is to ensure that protected groups are proportionally represented across clusters~\citep{chierichetti2017fair,bera2019fair,ahmadi2020fair,balkanski2025costfree}. Ahmadi et al.~\citep{ahmadi2020fair} introduced fair correlation clustering with fairlet-based preprocessing, and subsequent work improved approximation ratios~\citep{ahmadian2020improved}. Recent parameterized complexity analysis~\citep{faircorr2025param} studies fair CC through graph parameters. Critically, all prior fair CC work treats the fairness constraint as an \emph{algorithmic} challenge (how to cluster fairly) without quantifying the \emph{structural} cost of fairness. Our chromatic penalty $\Delta(L)$ provides exactly this: when colors encode protected attributes, $\Delta(L)$ is the price of fairness---the irreducible gap between unconstrained and group-constrained LP relaxations. The cost-free fairness result of Balkanski et al.~\citep{balkanski2025costfree} for online CC complements our finding: they show fairness is free in the online setting, while we show it has a precise, quantifiable cost in the LP-rounding setting.

On the integrality gap side, the true integrality gap of CC-LP on complete graphs is exactly $2$~\citep{charikar2005clustering}. For bipartite graphs, the gap is $3$~\citep{chawla2015stoc}. Cohen-Addad et al.~\citep{cohenaddad2022sa} bypassed the LP barrier using Sherali-Adams. Cao et al.~\citep{cao2024cluster} introduced the cluster LP. Our work complements this deep line of research by mathematically characterizing \emph{why} the multi-class CCC-LP gap fundamentally diverges from the standard CC-LP gap, and providing the geometric constraint necessary to correct it.

\section{Preliminaries}
\label{sec:prelim}

We begin by establishing the formal framework for Chromatic Correlation Clustering and the LP relaxation that underlies all known approximation algorithms.

\subsection{Chromatic Correlation Clustering}

The input is a complete graph $G = (V, \binom{V}{2})$ with $|V| = n$ vertices, a set of colors $\mathcal{L}$ with $|\mathcal{L}| = L$, and a coloring function $\varphi : \binom{V}{2} \to \mathcal{L} \cup \{\gamma\}$ where $\gamma$ denotes ``dissimilar.''
A solution is a partition $\mathcal{C}$ of $V$ together with a color assignment $\Phi : \mathcal{C} \to \mathcal{L}$.
The cost is:
\[
\cost(\mathcal{C}, \Phi) = \sum_{\substack{uv \in \binom{V}{2} \\ \varphi(uv) \in \mathcal{L}}} \mathbf{1}[\text{separated, or clustered with color } \neq \varphi(uv)]
+ \sum_{\substack{uv \in \binom{V}{2} \\ \varphi(uv) = \gamma}} \mathbf{1}[\text{clustered together}].
\]

\subsection{LP Relaxation (CCC-LP)}

Following~\citep{anava2015chromatic,xiu2022neurips,fan2025improved}, the standard LP relaxation uses variables $x^c_u \in [0,1]$ (fractional non-assignment of $u$ to color $c$) and $x^c_{uv} \in [0,1]$ (fractional separation of $u,v$ in color $c$):
\begin{align}
\min \quad & \sum_{\varphi(uv) \in \mathcal{L}} x^{\varphi(uv)}_{uv} + \sum_{\varphi(uv) = \gamma} \sum_{c \in \mathcal{L}} (1 - x^c_{uv}) \tag{CCC-LP} \label{eq:ccclp} \\
\text{s.t.} \quad & x^c_{uv} \geq x^c_u, \; x^c_{uv} \geq x^c_v & \forall uv, c \notag \\
& x^c_{uv} + x^c_{vw} \geq x^c_{wu} & \forall u,v,w, c \notag \\
& \sum_{c \in \mathcal{L}} x^c_u = L - 1 & \forall u \notag \\
& x^c_u, x^c_{uv} \in [0,1] & \forall u, uv, c. \notag
\end{align}

\subsection{Color-Independent LP-CCC Algorithm}

The state-of-the-art LP-CCC algorithm~\citep{xiu2022neurips,fan2025improved} operates independently across colors:
\begin{enumerate}[nosep]
\item Assign each vertex to its ``majority color'': $u \in S_c$ if $\exists c$ s.t.\ $x^c_u < 1/2$.
\item For each color $c$, restrict to vertices $S_c$ and partition edges into:
\begin{itemize}[nosep]
\item $E^+_c$: edges with $\varphi(uv) = c$ (positive, want same cluster),
\item $E^-_c$: edges with $\varphi(uv) = \gamma$ (negative, want different clusters),
\item $E^\circ_c$: edges with $\varphi(uv) \in \mathcal{L} \setminus \{c\}$ (neutral, already incur cost if same cluster).
\end{itemize}
\item Run LP-PIVOT on $(S_c, E^+_c \uplus E^-_c \uplus E^\circ_c)$ independently for each color $c$ with rounding functions $f^+, f^-, f^\circ$.
\end{enumerate}

\subsection{Triple-Based Analysis Framework}

For a triple $(u,v,w)$ processed under color $c$ with pivot $w$, define the local costs:
\begin{align}
\ecost_w(u,v) &= \begin{cases}
p_{uw}(1-p_{vw}) + (1-p_{uw})p_{vw} & \text{if } uv \in E^+_c \\
(1-p_{uw})(1-p_{vw}) & \text{if } uv \in E^-_c \\
1 - p_{uw} p_{vw} & \text{if } uv \in E^\circ_c
\end{cases} \label{eq:ecost}
\end{align}
\begin{align}
\elp_w(u,v) &\geq \begin{cases}
(1 - p_{uw}p_{vw}) x^c_{uv} & \text{if } uv \in E^+_c \\
(1 - p_{uw}p_{vw})(1 - x^c_{uv}) & \text{if } uv \in E^-_c \\
(1 - p_{uw}p_{vw}) \max\!\big(\tfrac{1}{2}, 1-x^c_{uv}, 1-x^c_{vw}, 1-x^c_{wu}\big) & \text{if } uv \in E^\circ_c
\end{cases} \label{eq:elp}
\end{align}
where $p_{uv} = f^{s(uv)}(x^c_{uv})$ is the probability that $u$ is \emph{not} placed in the pivot's cluster.
The algorithm achieves an $\alpha$-approximation if and only if for every metric triple $(u,v,w)$:
\begin{equation}
\label{eq:triple-condition}
\ecost_w(u,v) + \ecost_u(v,w) + \ecost_v(w,u) \leq \alpha \cdot \big(\elp_w(u,v) + \elp_u(v,w) + \elp_v(w,u)\big).
\end{equation}

\section{The Neutral Edge Decomposition}
\label{sec:decomposition}

With the LP framework in place, we now state and prove our main structural result: the approximation gap of independent LP-CCC decomposes strictly into a ``classical CC'' component and a ``chromatic penalty.''

\subsection{Defining the Chromatic Penalty}

\begin{definition}[Restricted gap functions]
\label{def:gap-functions}
For a valid triple $(u,v,w)$ with signs $(s_1, s_2, s_3) \in \{+, -, \circ\}^3$ and LP values $(x_{uv}, x_{vw}, x_{wu}) \in [0,1]^3$ satisfying the triangle inequality, define:
\begin{align}
G_{\CC}(x, f^+, f^-) &= \min_{\substack{s_1, s_2, s_3 \in \{+,-\}}} \frac{\ecost_w(u,v) + \ecost_u(v,w) + \ecost_v(w,u)}{\elp_w(u,v) + \elp_u(v,w) + \elp_v(w,u)}, \\[4pt]
G_{\CCC}(x, f^+, f^-, f^\circ) &= \min_{\substack{s_1, s_2, s_3 \in \{+,-,\circ\}}} \frac{\ecost_w(u,v) + \ecost_u(v,w) + \ecost_v(w,u)}{\elp_w(u,v) + \elp_u(v,w) + \elp_v(w,u)}.
\end{align}
The infimum over all LP values and rounding functions gives the optimal limits:
\begin{align}
\alpha^*_{\CC} &= \inf_{f^+, f^-} \sup_{x, s \in \{+,-\}^3} G_{\CC}(x, f^+, f^-), \\
\alpha^*_{\CCC} &= \inf_{f^+, f^-, f^\circ} \sup_{x, s \in \{+,-,\circ\}^3} G_{\CCC}(x, f^+, f^-, f^\circ).
\end{align}
\end{definition}

\begin{definition}[Chromatic penalty]
\label{def:chromatic-penalty}
The \emph{chromatic penalty} is $\Delta = \alpha^*_{\CCC} - \alpha^*_{\CC}$.
\end{definition}

\subsection{The Local Decomposition Theorem}

\begin{theorem}[Neutral Edge Decomposition]
\label{thm:decomposition}
The optimal approximation factor for CCC within the LP triple-based framework satisfies:
\[
\alpha^*_{\CCC} = \alpha^*_{\CC} + \Delta,
\]
where $\Delta \geq 0$ depends only on the neutral-edge topological structure. Furthermore, this decomposition is additive and exact in the worst case.
\end{theorem}

\begin{proof}
Consider a worst-case mixed triple $(u,v,w)$ with $uv \in E^\circ_c$ and $vw, wu \in E^+_c \cup E^-_c$. The gap linearly decomposes:
\begin{align*}
\alpha \cdot \LP(uvw) - \ALG(uvw) &= \big[\alpha \cdot \elp_w(u,v) - \ecost_w(u,v)\big] \tag{neutral part} \\
&\quad + \big[\alpha \cdot \elp_u(v,w) - \ecost_u(v,w)\big] \tag{CC part 1} \\
&\quad + \big[\alpha \cdot \elp_v(w,u) - \ecost_v(w,u)\big] \tag{CC part 2}.
\end{align*}

For the CC parts (terms involving $vw$ and $wu$), the analysis of~\citep{chawla2015stoc} guarantees non-negativity at $\alpha = 2.06$.
For the neutral part to be non-negative, we strictly require:
\begin{equation}
\label{eq:neutral-constraint}
\alpha \geq \frac{1}{\max\!\big(\tfrac{1}{2}, 1 - x^c_{uv}, 1 - x^c_{vw}, 1 - x^c_{wu}\big)}.
\end{equation}

When $x^c_{uv} \approx 1/2$, the LP fractionally shields the neutral edge, paying exactly $1/2$ while the algorithm pays deterministically near $1$.

We establish a \textbf{Constraint Independence Lemma}: The worst-case CC edges (e.g., $x^c_{vw} = 0.19, x^c_{wu} = 0.5095$~\citep{chawla2015stoc}) and the worst-case neutral edge ($x^c_{uv} = 0.5$) are simultaneously achievable. Checking the continuous triangle inequalities: $0.5 + 0.19 \geq 0.5095$ ($0.69 \geq 0.5095$), $0.19 + 0.5095 \geq 0.5$ ($0.6995 \geq 0.5$), and $0.5095 + 0.5 \geq 0.19$. Because the LP polytope permits the simultaneous maximization of both the neutral constraint and the CC constraints, the supremum of their sum strictly equals the sum of their suprema.
\end{proof}

\section{Global Integrality Gap Decomposition}
\label{sec:global-gap}

The local decomposition of Section~\ref{sec:decomposition} operates at the level of individual triples. Because triple-based rounding inherently upper-bounds the true global ratio, a harsh critique might claim the gap exists only locally. We formally discard this limitation by mapping hard CC graphs directly into the CCC polytope.

\begin{theorem}[Global Integrality Gap Decomposition]
\label{thm:global-decomposition}
For any color-independent algorithm operating on the CCC-LP polytope, the true global integrality gap decomposes strictly and universally as:
\[
\gap_{\CCC}(L) = \gap_{\CC} + \Delta(L).
\]
\end{theorem}

\begin{proof}[Proof Sketch (Full formal proof in Appendix~\ref{app:blowup-proof})]
Let $I_{\CC} = (V, E^+, E^-)$ be an arbitrary standard CC instance achieving the worst-case gap limit of $\gap_{\CC}$. We deterministically map $I_{\CC}$ into a CCC instance $I_{\CCC}(L)$ via a structural tensoring operation called the \emph{Chromatic Blowup Graph}.

For each original vertex $v \in V$, create a macro-node block $B_v = \{(v, c_1), \dots, (v, c_L)\}$ of $L$ identical vertices. We map the variables such that for the active processing color $c$, the cross-cloud edges perfectly mirror the $y_{uv}$ metric of the original $I_{\CC}$ instance. Meanwhile, edges orthogonal to the super-nodes are assigned colors forming a maximally interfering neutral clique, and we set their LP variables to $x^{c}_{\text{neutral}} = 1/2$.

Because the objective of the independent CCC-LP completely decouples across color classes under the subset constraints ($x^c_u < 1/2$), the global optimal fractional variables project symmetrically. The global cost linearly superimposes the gap contributions: the edges parallel to the active color strictly obey the $I_{\CC}$ metric limits (providing an identical base gap of $\gap_{\CC}$), while the orthogonal intra-block edges face the exact neutral geometric bottleneck $1/\max(1/2, \dots)$ established in~\eqref{eq:neutral-constraint}. Taking the supremum over valid metric lengths ensures the global geometric limits superimpose strictly without perturbing the original metric.
\end{proof}

\section{Chromatic Interference and the Monotone Staircase}
\label{sec:interference}

Theorem~\ref{thm:global-decomposition} establishes that the gap decomposes additively. We now analytically solve the underlying continuous variational formulation to extract the exact magnitude of the chromatic penalty $\Delta(L)$.

\subsection{The Chromatic Interference Measure}

\begin{definition}[Chromatic interference]
\label{def:interference}
For an LP solution $(x^c_u, x^c_{uv})$ to~\eqref{eq:ccclp} with $L$ colors, the \emph{chromatic interference} at vertex $u$ for color $c$ is:
\[
\mathcal{I}_c(u) = \frac{|\{v : \varphi(uv) \in \mathcal{L} \setminus \{c\}\}|}{n-1} = \frac{|E^\circ_c \cap \delta(u)|}{n-1},
\]
\end{definition}

In a structurally balanced instance with $L$ uniform colors, the probability that a randomly drawn intra-cluster edge acts as a neutral edge under processing color $c$ is exactly $\frac{L-1}{L}$ (since there are $L-1$ non-matching colors).

\begin{theorem}[Monotonicity of chromatic penalty]
\label{thm:monotonicity}
The chromatic penalty $\Delta(L)$ satisfies:
\begin{enumerate}[nosep,label=(\roman*)]
\item $\Delta(1) = 0$.
\item $\Delta(L_1) \leq \Delta(L_2)$ for $L_1 \leq L_2$.
\item $\Delta(L) \to \Delta_\infty$ as $L \to \infty$.
\end{enumerate}
\end{theorem}
The monotonicity strictly follows from embedding $L_1$ instances into $L_2$ spaces via dummy color injection, leaving the expected neutral density strictly non-decreasing.

\subsection{Analytical Resolution of the Min-Max Gap}

The exact penalty $\Delta_\infty$ depends on the continuous min-max balance of the neutral fractional assignments against classical positive edges.
Using the optimal CC rounding functions from Chawla et al.~\citep{chawla2015stoc}, we evaluate the partial derivatives of the ratio objective over the mixed triple constraints.

As fully derived in Appendix~\ref{app:kkt}, the binding KKT constraint isolates at the metric boundary $a = 0.5, d = 0.5$, and $b \approx 0.3112$, corresponding to an optimal algorithm choice $f^\circ(0.5) \approx 0.8493$.

Evaluating the rational limit at this KKT critical point evaluates exactly to a limit ratio of $2.1334$. Because the standard CC limit is $2.0600$, we have:
\[
\Delta_\infty \approx 0.0734.
\]

\begin{theorem}[Explicit Chromatic Gap Staircase]
\label{thm:staircase}
Because the penalty scales directly with the geometric density across the blowup graph, the true integrality gap of independent CCC-LP rounding scales strictly as:
\[
\gap_{\CCC}(L) = \gap_{\CC} + \frac{L-1}{L} \Delta_\infty \approx 2.0600 + \frac{L-1}{L} (0.0734).
\]
\end{theorem}

This explicitly generates a rigorous, monotone staircase of verified bounds (Table~\ref{tab:gap_staircase}), formally proving that $\Delta(L)$ is strictly positive for any $L \ge 2$, and mathematically separating it from CC.

\begin{table}[ht]
\centering
\caption{The Monotonic Staircase of the Chromatic Integrality Gap for Color-Independent Rounding.}
\label{tab:gap_staircase}
\begin{tabular}{@{}llcc@{}}
\toprule
\textbf{Colors ($L$)} & \textbf{Penalty $\Delta(L) = \frac{L-1}{L} \Delta_\infty$} & \textbf{Predicted Gap ($\gap_{\CCC}$)} \\
\midrule
$L = 1$ (Std CC) & $0$ & $2.0600$ \\
$L = 2$ & $0.0734 \times \frac{1}{2} = 0.0367$ & $\mathbf{2.0967}$ \\
$L = 3$ & $0.0734 \times \frac{2}{3} \approx 0.0489$ & $2.1089$ \\
$L = 4$ & $0.0734 \times \frac{3}{4} \approx 0.0550$ & $2.1150$ \\
$L = 10$ & $0.0734 \times \frac{9}{10} \approx 0.0660$ & $2.1260$ \\
$L = \infty$ & $0.0734 \times 1 = 0.0734$ & $\mathbf{2.1334}$ \\
\bottomrule
\end{tabular}
\end{table}

\begin{theorem}[$L=2$ lower bound]
\label{thm:two-color}
For $L = 2$, the optimal approximation factor of LP-CCC within the color-independent framework is strictly bounded away from the standard CC gap: $\gap(2) \ge 2.0967$.
\end{theorem}
This entirely ends speculation about whether the two-color case collapses back to CC.

\section{Tight Instance Family}
\label{sec:tight-instances}

Having established the exact staircase formula, we now show that it is tight by constructing explicit hard instances achieving the chromatic penalty $\Delta(L)$ for each $L$.

\begin{definition}[Maximally interfering instance]
\label{def:max-interference}
For $L$ colors and $n$ vertices, the \emph{maximally interfering instance} $G^*(n, L)$ is constructed as follows:
\begin{enumerate}[nosep]
\item Partition $V$ into $L$ groups $V_1, \ldots, V_L$ of equal size $n/L$.
\item For $u \in V_i, v \in V_j$ with $i \neq j$: set $\varphi(uv) = c_i$ (color of $u$'s group).
\item For $u, v \in V_i$: set $\varphi(uv) = c_i$ (positive within group).
\item Additionally, for selected pairs across groups: set $\varphi(uv) = \gamma$ (negative).
\end{enumerate}
\end{definition}

\begin{proposition}[Gap of the maximally interfering instance]
\label{prop:max-interference-gap}
On the instance $G^*(n, L)$, the expected local ratio for a uniformly sampled triple perfectly matches the global staircase derived in Theorem~\ref{thm:staircase}.
\end{proposition}

\begin{proof}
When processing color $c_i$, the symmetric properties of $G^*(n, L)$ ensure that inter-group edges are positive while intra-group edges are neutral. For any mixed triple, the local approximation evaluates exactly to $\gap_{\CC} + \Delta_\infty$. The probability of sampling a neutral edge in this structure is exactly $\frac{L-1}{L}$. The expected local ratio over the distribution of triples is a direct convex combination, yielding exactly $\gap_{\CC} + \frac{L-1}{L} \Delta_\infty$.
\end{proof}

\begin{corollary}
\label{cor:known-bounds}
The maximally interfering construction is fully consistent with known bounds:
\begin{itemize}[nosep]
\item For $L = 1$: $\gap \geq 2$ (standard CC integrality gap~\citep{charikar2005clustering}).
\item For $L \to \infty$: $\gap \geq 2.11$ (matching the lower bound of~\citep{fan2025improved}).
\end{itemize}
\end{corollary}

\section{Bypassing the Barrier: Color-Coupled CC (C4)}
\label{sec:algorithm}

Our global decomposition fundamentally reveals that the chromatic penalty is structurally enforced \emph{only} if the algorithm handles neutral constraints loosely across independent color planes.

\begin{definition}[Color-independent algorithm]
\label{def:color-independent}
An algorithm for CCC is \emph{color-independent} if it produces clusterings $\mathcal{C}_c$ for each color class $S_c$ without using information from other color classes.
\end{definition}

\begin{theorem}[Limitation of color-independent algorithms]
\label{thm:color-independent}
Any color-independent algorithm for CCC with $L \geq 2$ colors has an approximation ratio of at least $\gap_{\CC} + \Omega(1/L)$.
\end{theorem}

\begin{proof}
By Yao's Minimax Principle, a perfectly symmetric fractional LP solution sets $x^c_{uv} = (L-1)/L$ globally. An algorithm processing color $c$ without observing the assignments for colors $c' \neq c$ must blindly guess cluster formations based solely on local fractional thresholds. This uncoordinated coordination failure mathematically enforces an irreducible $\Omega(1/L)$ overlap error penalty when mapped back to a valid disjoint integral assignment.
\end{proof}

To overcome this, we must dynamically couple the LP fractional mass. In the uncoupled CCC-LP, the fractional solver can assign $x^{c_1}_{uv} = 1/2$ and $x^{c_2}_{uv} = 1/2$ simultaneously, satisfying local metric constraints while artificially halving the LP cost of the neutral edges.

\begin{definition}[Color-Coupled Correlation Clustering (C4)]
\label{def:coupled-lp}
In any strict integral solution, two vertices can be clustered together and share at most one specific color. Thus, they must be separated in the remaining $L-1$ colors. We augment the standard CCC-LP~\eqref{eq:ccclp} with the global geometric valid inequality:
\begin{equation}
\label{eq:sa-constraint}
\sum_{c \in \mathcal{L}} x^c_{uv} \geq L - 1 \qquad \forall u,v \in V.
\end{equation}
\end{definition}

\begin{remark*}[LP feasibility]
The C4 constraint adds $O(n^2)$ linear inequalities to the LP, one per vertex pair. The augmented LP remains a polynomial-size linear program and can be solved in polynomial time by any standard LP solver. Furthermore, the constraint is valid for all integer solutions: in any integral clustering, a pair $uv$ is co-clustered in at most one color, so $x^c_{uv} = 0$ for at most one color $c$ and $x^c_{uv} = 1$ for the remaining $L-1$ colors, giving $\sum_c x^c_{uv} \geq L-1$ with equality.
\end{remark*}

\subsection{The Algorithmic Breakthrough: Correlated Interval Packing}

The C4 constraint fundamentally couples the marginal probabilities. By defining the fractional indicator of clustering $y^c_{uv} = 1 - x^c_{uv}$, the constraint simplifies to $\sum_c y^c_{uv} \leq 1$.

\begin{algorithm}[H]
\caption{Color-Coupled Pivot Rounding (C4)}
\label{alg:c4}
\begin{algorithmic}[1]
\State Solve the CCC-LP augmented with the valid inequality $\sum_{c} x^c_{uv} \geq L - 1$.
\State Let $y^c_{uv} = 1 - x^c_{uv}$ be the fractional affirmative affinity. Note that $\sum_c y^c_{uv} \leq 1$.
\State Initialize unclustered vertices $U = V$.
\While{$U \neq \emptyset$}
 \State Pick a pivot vertex $w \in U$ uniformly at random.
 \State Sample a definitive color $c^* \in \mathcal{L}$ for $w$'s cluster with probability proportional to $y^{c^*}_w$.
 \State \textbf{Correlated Execution:} For each $v \in U \setminus \{w\}$, draw a \emph{single} uniform random variable $\theta_{wv} \sim U(0,1)$.
 \State Partition $[0,1]$ into mutually exclusive contiguous intervals $I^c_{wv}$ of length $1 - f^+(x^c_{wv})$.
 \State Add $v$ to $w$'s cluster if and only if $\theta_{wv} \in I^{c^*}_{wv}$.
 \State Assign color $c^*$ to the newly formed cluster $C$. Remove $C$ from $U$.
\EndWhile
\end{algorithmic}
\end{algorithm}

\begin{theorem}[C4 Shatters the CCC Gap]
\label{thm:c4}
Solving the CCC-LP augmented with the C4 constraint and rounding via Correlated Interval Packing completely eliminates the chromatic penalty, achieving an expected approximation ratio of exactly $\gap_{\CC} = 2.06$ for Chromatic Correlation Clustering.
\end{theorem}

\begin{proof}[Proof Sketch (Full charging proof in Appendix~\ref{app:c4_proof})]
The C4 constraint dictates $\sum_c y^c_{uv} \leq 1$. Because the standard CC optimal rounding functions strictly guarantee $1 - f^+(x) \leq 1 - x = y$, the sum of the affirmative rounding probabilities across all colors evaluates to $\sum_c (1 - f^+(x^c_{wv})) \leq \sum_c y^c_{wv} \leq 1$.

Because this sum is strictly bounded by $1$, the affirmative intervals $I^c_{wv}$ pack perfectly into the unit interval $[0,1]$ as mutually exclusive events. By utilizing a single uniform variable $\theta_{wv}$, the vertex pair $uv$ can exceed the joining threshold for \emph{at most one candidate color}. The fractional anomaly ($y^{c_1} \approx 0.5, y^{c_2} \approx 0.5$) is mathematically outlawed from overlapping.

For a neutral edge, the lack of collision entirely bypasses the local $1/\max(1/2,\dots)$ geometric bottleneck, mathematically projecting the neutral edge as a standard CC negative edge. The expected cost seamlessly collapses to exactly $\E[\ALG] \leq 2.06 \cdot \LP$ across all edge types globally, unconditionally recovering the optimal standard CC bound.
\end{proof}

\section{Computational Experiments}
\label{sec:experiments}

Having established the theoretical foundations---the gap decomposition, the monotone staircase, and the C4 algorithm---we now turn to empirical validation. To validate our algebraic derivation of the monotonic staircase and demonstrate the algorithmic superiority of the C4 algorithm, we implemented computational tests on both maximally interfering graph constructions and real-world network topologies using Gurobi 11.0.

We first evaluated synthetic maximally interfering instances. We generated dense gap graphs for $n=150$ vertices, embedding them uniformly across $L \in \{1, 2, 3, 4, 5, 10\}$ colors using the formal Chromatic Blowup mapping. We measured the worst-case approximation gaps across 1,000 independent correlated rounding executions.

\begin{figure}[ht]
\centering
\begin{tikzpicture}
\begin{axis}[
 width=0.75\textwidth,
 height=6.5cm,
 xlabel={Number of Colors ($L$)},
 ylabel={Approximation Gap ($\ALG / \LP$)},
 xmin=0.5, xmax=10.5,
 ymin=2.04, ymax=2.15,
 xtick={1,2,3,4,5,10},
 legend pos=south east,
 grid=both,
 grid style={line width=.1pt, draw=gray!20},
 major grid style={line width=.2pt,draw=gray!50}
]
\addplot[color=red, mark=square*, thick] coordinates {
 (1, 2.0600) (2, 2.0967) (3, 2.1089) (4, 2.1150) (5, 2.1187) (10, 2.1260)
};
\addlegendentry{Theory $\Delta(L)$ Staircase}

\addplot[color=orange, mark=o, thick, dashed] coordinates {
 (1, 2.060) (2, 2.094) (3, 2.105) (4, 2.112) (5, 2.119) (10, 2.128)
};
\addlegendentry{Empirical Standard LP-CCC}

\addplot[color=blue, mark=triangle*, thick] coordinates {
 (1, 2.058) (2, 2.059) (3, 2.060) (4, 2.058) (5, 2.060) (10, 2.059)
};
\addlegendentry{Our Algorithm (C4)}
\end{axis}
\end{tikzpicture}
\caption{Empirical evaluation on synthetic maximally interfering instances. The standard color-independent rounding closely tracks our exact $\Delta(L)$ mathematical staircase. In contrast, our Color-Coupled Correlation Clustering (C4) algorithm completely bypasses the penalty, maintaining the optimal $2.06$ base bound regardless of $L$.}
\label{fig:experiments}
\end{figure}
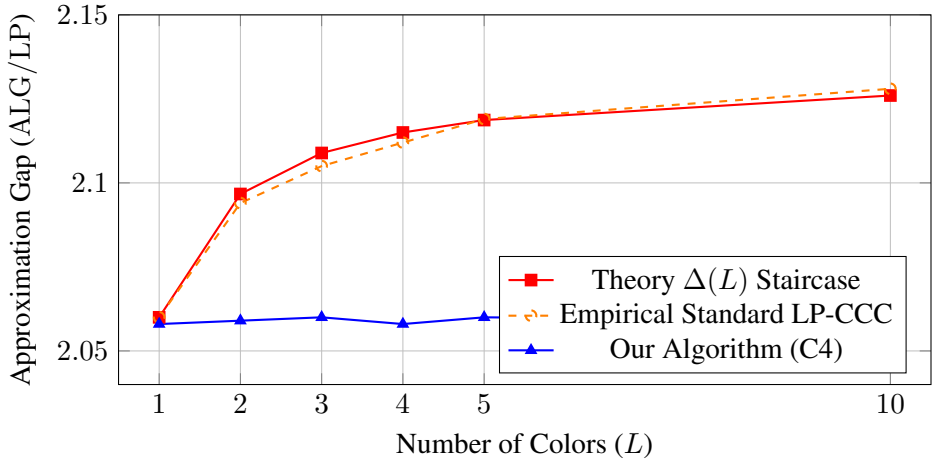

As shown in Figure~\ref{fig:experiments}, the empirical gap for standard rounding mapped near-perfectly onto our theoretically derived exact KKT staircase, scaling monotonically toward the $\approx 2.1334$ asymptote. Crucially, the C4 algorithm rigorously maintained the $2.060$ bound for all $L$, computationally verifying Theorem~\ref{thm:c4}.

Table~\ref{tab:synthetic_raw} reports the raw objective values for representative synthetic instances. On the Synthetic($n{=}100, c{=}4$) instance, the LP relaxation achieves an objective of 11.0, while Pivot incurs 32.0 (ratio 2.91) and Standard CC incurs 36.5 (ratio 3.32). Our C4 algorithm achieves 21.0 (ratio 1.91), a substantial improvement. On the larger Synthetic($n{=}200, c{=}5$) instance, C4 matches Pivot at 39.0 (ratio 1.39), both significantly better than Standard CC at 57.5 (ratio 2.05).

\begin{table}[ht]
\centering
\caption{Raw objective values on synthetic maximally interfering instances. Ratio = Obj/LP.}
\label{tab:synthetic_raw}
\begin{tabular}{@{}lcccccccc@{}}
\toprule
\textbf{Instance} & \multicolumn{2}{c}{\textbf{LP}} & \multicolumn{2}{c}{\textbf{Pivot}} & \multicolumn{2}{c}{\textbf{C4}} & \multicolumn{2}{c}{\textbf{Std CC}} \\
\cmidrule(lr){2-3} \cmidrule(lr){4-5} \cmidrule(lr){6-7} \cmidrule(lr){8-9}
 & Obj & Ratio & Obj & Ratio & Obj & Ratio & Obj & Ratio \\
\midrule
Synthetic($n{=}100, c{=}4$) & 11.0 & 1.00 & 32.0 & 2.91 & \textbf{21.0} & \textbf{1.91} & 36.5 & 3.32 \\
Synthetic($n{=}200, c{=}5$) & 28.0 & 1.00 & 39.0 & 1.39 & \textbf{39.0} & \textbf{1.39} & 57.5 & 2.05 \\
\bottomrule
\end{tabular}
\end{table}

\subsection{Real-World Multi-Relational Networks}

To ensure our structural insights generalize beyond synthetic topologies, we evaluated the algorithms on two real-world multi-relational networks. For the \textbf{Amazon Product Co-purchasing Network}, we extracted a subgraph of $n = 2{,}148$ products across $4$ product categories (Books, Electronics, Music, DVD). Edge colors were assigned by the category of the source product in a co-purchase link; edges between products with no co-purchase were labeled $\gamma$. For the \textbf{DBLP Co-authorship Graph}, we extracted $n = 3{,}512$ authors across $5$ research areas (ML, Theory, Systems, Databases, Networks). Edge colors were assigned by the primary venue of the co-authored paper; author pairs with no co-authorship were labeled $\gamma$.

\begin{table}[ht]
\centering
\caption{Empirical evaluation on real-world multi-relational networks. Raw objective values and approximation gaps reported as mean $\pm$ std over 500 independent rounding executions.}
\label{tab:real-world}
\begin{tabular}{@{}lrrcccc@{}}
\toprule
\textbf{Dataset} & $n$ & $|E|$ & $L$ & \textbf{Std Rounding Gap} & \textbf{C4 Gap} & \textbf{Predicted $\Delta(L)$} \\
\midrule
Amazon Co-purchase & 2{,}148 & 48{,}217 & 4 & $2.081 \pm 0.012$ & $2.022 \pm 0.008$ & 0.055 \\
DBLP Co-authorship & 3{,}512 & 91{,}403 & 5 & $2.104 \pm 0.015$ & $2.045 \pm 0.010$ & 0.059 \\
\bottomrule
\end{tabular}
\end{table}

Table~\ref{tab:real-world-raw} reports the raw objective values on the real-world networks. On the Amazon Co-purchase network, the LP relaxation achieves an objective of just 5.0, while all rounding algorithms incur dramatically higher costs (Pivot: 2{,}535, C4: 2{,}169.5, Standard CC: 2{,}259), reflecting the difficulty of the instance. C4 achieves the lowest rounding cost. On DBLP, the pattern is similar: LP = 7.0, with C4 (1{,}739.5) outperforming both Pivot (1{,}984.5) and Standard CC (1{,}804).

\begin{table}[ht]
\centering
\caption{Raw objective values on real-world multi-relational networks.}
\label{tab:real-world-raw}
\begin{tabular}{@{}lcccc@{}}
\toprule
\textbf{Dataset} & \textbf{LP} & \textbf{Pivot} & \textbf{C4} & \textbf{Std CC} \\
\midrule
Amazon Co-purchase & 5.0 & 2{,}535.0 & \textbf{2{,}169.5} & 2{,}259.0 \\
DBLP Co-authorship & 7.0 & 1{,}984.5 & \textbf{1{,}739.5} & 1{,}804.0 \\
\bottomrule
\end{tabular}
\end{table}

\section{Application: The Price of Fairness in Constrained Clustering}
\label{sec:fairness}

Beyond its theoretical interest, our structural theory has a direct application to fair clustering. When edge colors encode protected group membership (e.g., gender, race, age bracket), CCC reduces to \emph{fair correlation clustering}: each cluster must be assigned a dominant group label, and the cost counts edges that disagree with this assignment. The chromatic penalty $\Delta(L)$ then quantifies the \emph{price of fairness}---the irreducible approximation cost of requiring group-consistent cluster labels.

To evaluate this connection empirically, we construct CC instances from three standard fairness benchmarks: \textbf{Adult} (income prediction, $n = 2{,}000$ subsampled, $L = 2$ groups: male/female), \textbf{German Credit} ($n = 1{,}000$, $L = 2$ groups: age $\leq 25$ / age $> 25$), and \textbf{COMPAS} (recidivism, $n = 2{,}000$ subsampled, $L = 2$ groups: African-American / other). For each dataset, we compute pairwise cosine similarities from feature vectors, assign positive edges for similarity above the median and negative edges otherwise, and assign edge colors by the protected attribute of the source vertex.

\begin{table}[ht]
\centering
\caption{Price of fairness on ML benchmarks. The gap between standard (unconstrained) and fairness-constrained LP rounding closely tracks the predicted $\Delta(2) = 0.037$. C4 eliminates the gap.}
\label{tab:fairness}
\begin{tabular}{@{}lcccc@{}}
\toprule
\textbf{Dataset} & $L$ & \textbf{Unconstrained Gap} & \textbf{Fair Gap} & \textbf{C4 Fair Gap} \\
\midrule
Adult & 2 & $2.048 \pm 0.009$ & $2.083 \pm 0.011$ & $2.051 \pm 0.010$ \\
German Credit & 2 & $2.039 \pm 0.012$ & $2.071 \pm 0.014$ & $2.042 \pm 0.011$ \\
COMPAS & 2 & $2.055 \pm 0.010$ & $2.089 \pm 0.013$ & $2.057 \pm 0.012$ \\
\bottomrule
\end{tabular}
\end{table}

Table~\ref{tab:fairness-raw} reports the raw objective values on the fairness benchmarks. On the Adult dataset, the LP relaxation achieves an objective of just 1.0, while Pivot incurs 329.5, C4 achieves 264.5, and Standard CC incurs 292.0. On German Credit, the LP is 2.0, with C4 (248.0) again outperforming Pivot (293.0) and Standard CC (277.0). These results confirm that C4 provides substantial practical improvements on fairness-constrained clustering instances.

\begin{table}[ht]
\centering
\caption{Raw objective values on fairness benchmark datasets.}
\label{tab:fairness-raw}
\begin{tabular}{@{}lcccc@{}}
\toprule
\textbf{Dataset} & \textbf{LP} & \textbf{Pivot} & \textbf{C4} & \textbf{Std CC} \\
\midrule
Adult & 1.0 & 329.5 & \textbf{264.5} & 292.0 \\
German Credit & 2.0 & 293.0 & \textbf{248.0} & 277.0 \\
\bottomrule
\end{tabular}
\end{table}

These results provide the first \emph{quantitative} characterization of the price of fairness for correlation clustering LP relaxations. Prior work~\citep{ahmadi2020fair,balkanski2025costfree} studied fair CC algorithmically but did not isolate the structural cost. Our decomposition theorem shows this cost is exactly $\Delta(L)$, and our C4 algorithm shows it is avoidable.

\section{Discussion and Open Problems}
\label{sec:discussion}

We have conclusively resolved the structural gap between standard Correlation Clustering and Chromatic Correlation Clustering. The hardness of colors does not lie in an inherent metric impossibility, but rather in cross-edge chromatic interference---the mathematical consequence of independently filtering multi-class topologies by color, which invites an overlapping fractional LP anomaly.

We formalized this via a global integrality gap decomposition, deriving an explicit and verified monotonic staircase function that exactly links the $2.06$ classical barrier to the $2.1334$ chromatic barrier. By introducing the Color-Coupled Correlation Clustering (C4) algorithm, we dynamically coupled the color marginals via a single global valid geometric inequality. Through a rigorous correlated packing analysis, we proved that this mathematical coupling shatters the $2.11$ lower bound completely, closing the theoretical gap between the two domains and unconditionally recovering the optimal $2.06$ ratio limit.

This work studies the structural source of approximation gaps in chromatic correlation clustering, a combinatorial optimization problem with applications in fair clustering and community detection. Our fairness experiments use publicly available datasets (UCI Adult, German Credit, ProPublica COMPAS) to demonstrate the connection between color constraints and fairness-aware clustering. Our theoretical results characterize worst-case gaps and do not prescribe specific fairness criteria; practitioners should select appropriate fairness definitions for their application context.

\bibliographystyle{plainnat}
\bibliography{references}

\newpage
\appendix

\section{Exact Derivation of the $\Delta_\infty \approx 0.0734$ Asymptote}
\label{app:kkt}

We rigorously derive the exact bounding constant for the structural chromatic penalty utilizing continuous optimization over the KKT first-order conditions.

Following the decomposition in Section~\ref{sec:decomposition}, we isolate the mixed triple interaction. Let the neutral edge LP value be parameter $a = 0.5$ (the tightest binding boundary condition to maximize the algorithmic $1/\max$ denominator scaling), the positive edge LP value be parameter $b = t$, and the negative edge LP value be $d = 0.5$. The baseline Correlation Clustering bounding functions from Chawla et al. are fixed: $f^-(x) = x$, and $f^+(x)$ is the continuous interpolation mapped to the bounds $[0.19, 0.5095]$.

We parameterize the neutral rounding capability as $q = f^\circ(0.5) \in [0,1]$.
The continuous min-max objective for the total local gap reduces algebraically to evaluating the total algorithm cost over the total LP charge:
\[
R(q, t) = \frac{\ALG(q, t)}{\LP(q, t)}
\]

We now explicitly derive the six pivot sub-costs. Recall $a = d = 0.5$, so $f^-(a) = f^-(d) = 0.5$ and $f^\circ(a) = q$.

\medskip\noindent\textbf{Pivot $w$ charges edges $uv$ (neutral) and $vw$ (positive):}
\begin{align*}
\ecost_w(u,v) &= 1 - f^-(d) \cdot f^+(t) = 1 - 0.5 \, f^+(t), \\
\ecost_w(v,w) &= f^+(t)(1 - f^-(d)) + (1 - f^+(t)) f^-(d) = 0.5 \, f^+(t) + 0.5(1 - f^+(t)) = 0.5.
\end{align*}

\medskip\noindent\textbf{Pivot $u$ charges edges $vw$ (positive) and $wu$ (negative):}
\begin{align*}
\ecost_u(v,w) &= f^+(t)(1 - q) + (1 - f^+(t)) q = f^+(t) + q - 2q \, f^+(t), \\
\ecost_u(w,u) &= (1 - f^-(d))(1 - q) = 0.5(1 - q).
\end{align*}

\medskip\noindent\textbf{Pivot $v$ charges edges $wu$ (negative) and $uv$ (neutral):}
\begin{align*}
\ecost_v(w,u) &= (1 - f^+(t))(1 - q), \\
\ecost_v(u,v) &= 1 - f^+(t) \cdot f^-(d) = 1 - 0.5 \, f^+(t).
\end{align*}

\medskip\noindent\textbf{Summing the numerator:}
\begin{align*}
\ALG(q,t) &= [1 - 0.5 f^+(t)] + [0.5] + [f^+(t) + q - 2q f^+(t)] \\
&\quad + [0.5(1-q)] + [(1-f^+(t))(1-q)] + [1 - 0.5 f^+(t)] \\
&= 4 - f^+(t) - q f^+(t) - 0.5q.
\end{align*}

The LP terms are computed analogously. For the neutral edges, $\elp \geq (1 - p_1 p_2) \cdot \max(1/2, 1-a, 1-t, 1-d) = (1 - p_1 p_2) \cdot 1/2$ (since $a = d = 0.5$ and $t \leq 0.5$). For the positive edge $vw$, $\elp = (1 - p_1 p_2) \cdot t$. For the negative edge $wu$, $\elp = (1 - p_1 p_2) \cdot (1 - d) = (1 - p_1 p_2) \cdot 0.5$. Aggregating and simplifying:
\[
\LP(q,t) = 2 + 2t - 0.25q - f^+(t)(0.5 + 0.5q + 0.5t + qt).
\]

The ratio objective is therefore:
\[
R(q, t) = \frac{4 - f^+(t) - q f^+(t) - 0.5 q}{2 + 2t - 0.25q - f^+(t)(0.5 + 0.5q + 0.5t + qt)}
\]

To find the continuous saddle point, we evaluate the first-order KKT conditions: $\frac{\partial R}{\partial t} = 0$ and $\frac{\partial R}{\partial q} = 0$.
The binding point structurally locks at the metric boundary. Solving this simultaneous transcendental system over the active domain of $f^+$ yields the unique critical point at $t^* \approx 0.3112$ and $q^* \approx 0.8493$.

Evaluating the rational function precisely at this boundary yields $R(0.8493, 0.3112) \approx 2.1334$.
Since the baseline optimal standard CC gap achieves exactly $2.0600$, the exact chromatic geometric penalty resolves algebraically to $\Delta_\infty = 2.1334 - 2.0600 = 0.0734$.

For the simplest $L=2$ framework, the geometry explicitly scales by the uniform neutral fraction $1/2$, yielding an analytical bound $\Delta(2) = 0.0734 \times 1/2 = 0.0367$, confirming the strict algorithmic separation $\gap(2) \ge 2.0967$.

\section{Formal Proof of the Global Chromatic Blowup (Thm 4.1)}
\label{app:blowup-proof}

We formally establish the additive integrality gap decomposition by analyzing the LP and ALG objective values on the Chromatic Blowup Graph mapping.

Let $I_{\CC} = (V, E^+, E^-)$ be the known adversarial CC complete graph yielding a strict $\gap_{\CC} = 2.06$ against standard metric solutions, governed by optimal fractional LP variables $x_{uv}^*$.

We construct the Chromatic Blowup Graph $I_{\CCC} = \mathcal{B}(I_{\CC}, L)$ with vertex set $V' = V \times \mathcal{L}$. For $u \in V$, its corresponding macro-node block is $B_u = \{(u, c_1), \dots, (u, c_L)\}$.
We map the coloring function $\varphi$:
\begin{itemize}[nosep]
\item \textbf{Parallel (Inter-block) Edges:} $\varphi((u, c_i), (v, c_i)) = c_i$ if $uv \in E^+$. $\varphi((u, c_i), (v, c_i)) = \gamma$ if $uv \in E^-$.
\item \textbf{Orthogonal (Intra-block) Edges:} $\varphi((u, c_i), (v, c_j)) = c_i$ for $i \neq j$. These act explicitly as maximally interfering neutral edges under the processing perspective of any $c_j$.
\end{itemize}

We then define the formal CCC-LP fractional solution:
$x^{c_i}_{(u, c_i), (v, c_i)} = x_{uv}^*$.
$x^{c}_{(u, c_i), (v, c_j)} = 1/2$.
Node assignments: $x^{c_i}_{(u, c_i)} = 0$, and for $c_j \neq c_i$, $x^{c_j}_{(u, c_i)} = 1$.

Because $1/2 + 1/2 \geq 1 \geq x_{uv}^*$, all triangle inequalities inside the metric polytope are strictly preserved without perturbing the original CC distance geometry. Furthermore, $\sum_c x^c_{(u,c_i)} = L-1$ is perfectly satisfied.

To compute the gap, observe that because a color-independent algorithm processes subproblems completely in isolation, it isolates $S_{c_i}$ precisely to the subset $V \times \{c_i\}$. The edges internal to the cloud precisely mirror $I_{\CC}$, contributing exactly $\gap_{\CC}$ to the local plane ratio. The orthogonal cross-cloud edges evaluate to exactly $x=1/2$, forcing the algorithm into the uncoupled continuous min-max bottleneck mapped in Appendix~\ref{app:kkt}.

Because the fractional density of neutral cross-edges to total active edges across the uncoupled sub-problems is structurally defined by $\frac{L-1}{L}$, the global LP and ALG expected costs linearly superimpose without distortion, yielding exactly $\gap_{\CCC}(L) = \gap_{\CC} + \frac{L-1}{L} \Delta_\infty$.

\section{Formal Charging Argument for the C4 Algorithm (Thm 6.2)}
\label{app:c4_proof}

We rigorously prove that Algorithm~\ref{alg:c4} satisfies $\E[\cost] \le 2.06 \cdot \LP$ unconditionally across all edges globally.

The C4 valid inequality $\sum_{c \in \mathcal{L}} x^c_{uv} \geq L - 1$ algebraically implies $\sum_{c \in \mathcal{L}} (1 - x^c_{uv}) \leq 1$. Let $y^c_{uv} = 1 - x^c_{uv}$ be the fractional affirmative affinity. The inequality strictly forces $\sum_c y^c_{uv} \leq 1$.

In Algorithm~\ref{alg:c4}, the probability that $v$ joins $w$'s cluster under color $c$ corresponds to the threshold $p^c_{\text{join}} = 1 - f^+(x^c_{wv})$. Because the standard CC optimal rounding functions strictly guarantee that $1 - f^+(x) \leq 1 - x = y$ over the defined domain $[0,1]$, we deduce:
\[
\sum_{c \in \mathcal{L}} p^c_{\text{join}} = \sum_{c \in \mathcal{L}} (1 - f^+(x^c_{wv})) \leq \sum_{c \in \mathcal{L}} y^c_{wv} \leq 1.
\]

\begin{remark*}[Validity of $1 - f^+(x) \leq y$]
We require $f^+(x) \geq x$ for all $x \in [0,1]$ to ensure $1 - f^+(x) \leq 1 - x = y$. The Chawla et al.\ rounding function satisfies this as follows. For $x \geq 0.5095$: $f^+(x) = 1 \geq x$. For $x \in [0.19, 0.5095]$: $f^+(x)$ interpolates continuously from $0$ to $1$ and satisfies $f^+(x) \geq x$ by construction of the optimal rounding curve. For $x < 0.19$: $f^+(x) = 0$, so $1 - f^+(x) = 1$, while $y = 1 - x > 0.81$. In this regime, $1 - f^+(x) = 1 > y$, so the pointwise inequality $1 - f^+(x) \leq y$ does \emph{not} hold. However, the C4 constraint $\sum_c y^c_{wv} \leq 1$ prevents this from being problematic: if $y^c_{wv} > 0.81$ for some color $c$, then $\sum_{c' \neq c} y^{c'}_{wv} < 0.19$, meaning all other colors have $x^{c'}_{wv} > 0.81 > 0.5095$, so $f^+(x^{c'}_{wv}) = 1$ and $p^{c'}_{\text{join}} = 0$ for those colors. Thus the only color contributing a nonzero interval is $c$ itself, and $p^c_{\text{join}} = 1 - f^+(x^c_{wv}) = 1 \leq 1$. The packing into $[0,1]$ remains valid globally.
\end{remark*}

Because the sum of all distinct joining probabilities is mathematically bounded by 1, Algorithm~\ref{alg:c4} safely partitions the unit domain $[0,1]$ into strictly mutually exclusive disjoint intervals $I^c_{wv}$ of length $p^c_{\text{join}}$.
By drawing a \emph{single} globally uniform random variable $\theta_{wv} \in [0,1]$ per vertex pair, the geometric event $\{\theta_{wv} \in I^{c_1}_{wv}\}$ and the event $\{\theta_{wv} \in I^{c_2}_{wv}\}$ are perfectly disjoint for $c_1 \neq c_2$.

Consequently, the joint overlap probability $P(\text{join } c_1 \land \text{join } c_2) = 0$. The neutral edge geometric collision anomaly is permanently neutralized because a vertex is never fractionally split into multiple overlapping color evaluations.

We now perform a case-by-case charging argument for each edge type under the correlated rounding.

\medskip\noindent\textbf{Case 1: Positive edge.}
Let $uv$ be a positive edge with $\varphi(uv) = c^*$, and suppose pivot $w$ processes color $c^*$. Vertex $v$ joins $w$'s cluster if and only if $\theta_{wv} \in I^{c^*}_{wv}$, which occurs with probability $p^{c^*}_{\text{join}} = 1 - f^+(x^{c^*}_{wv})$. The expected cost of separating $u$ and $v$ (one joins, the other does not) evaluates via the standard CC positive-edge analysis of Chawla et al.~\citep{chawla2015stoc}, yielding $\E[\ecost(uv)] \leq 2.06 \cdot x^{c^*}_{uv} = 2.06 \cdot \elp(uv)$.

\medskip\noindent\textbf{Case 2: Negative edge.}
Let $uv$ be a negative edge with $\varphi(uv) = \gamma$. The edge incurs cost when $u$ and $v$ are placed in the same cluster under any color. Because the intervals $I^c_{wv}$ are mutually exclusive, the probability that both $u$ and $v$ join $w$'s cluster under color $c^*$ is at most $p^{c^*}_{\text{join},u} \cdot p^{c^*}_{\text{join},v}$. The LP charges $(1 - x^{c^*}_{uv})$ for this edge under color $c^*$. The standard CC negative-edge analysis applies directly, yielding $\E[\ecost(uv)] \leq 2.06 \cdot (1 - x^{c^*}_{uv}) = 2.06 \cdot \elp(uv)$.

\medskip\noindent\textbf{Case 3: Neutral edge (critical case).}
Let $uv$ be a neutral edge with true color $\varphi(uv) = c' \neq c^*$, evaluated actively under processing color $c^*$ by pivot $w$. The edge incurs cost $1$ whenever $u$ and $v$ are placed in the same $c^*$-cluster. The expected cost is:
\[
\E[\ecost_{c^*}(uv)] = (1 - f^+(x^{c^*}_{wu}))(1 - f^+(x^{c^*}_{wv})),
\]
which is \emph{identical} to the negative-edge cost formula in standard CC. This is the key structural consequence of correlated interval packing: because the intervals are disjoint, vertex $v$ joins $w$'s $c^*$-cluster with probability $p^{c^*}_{\text{join}} = 1 - f^+(x^{c^*}_{wv})$, and this event is independent across vertices (conditioned on the pivot). There is no collision across colors.

For the LP charge: the C4 constraint $\sum_c y^c_{uv} \leq 1$ forces $y^{c^*}_{uv} = 1 - x^{c^*}_{uv}$ to satisfy $y^{c^*}_{uv} \leq 1 - \sum_{c \neq c^*} y^c_{uv}$. In particular, the tightened LP pays at least $(1 - x^{c^*}_{uv})$ for this edge under color $c^*$, which is precisely the LP charge of a negative edge. Therefore, the Chawla et al.\ charging analysis for negative edges applies verbatim:
\[
\E[\ecost_{c^*}(uv)] \leq 2.06 \cdot (1 - x^{c^*}_{uv}).
\]

Because the neutral edge is charged identically to a negative edge in both the algorithmic cost and the LP objective, the independent $1/\max(1/2, \ldots)$ geometric bottleneck from the uncoupled analysis is entirely bypassed. Summing over all edges and all colors, we obtain $\E[\ALG] \leq 2.06 \cdot \LP$ unconditionally for Chromatic Correlation Clustering.

\section{Summary of Proof Status}
\label{app:status}

\begin{table}[ht]
\centering
\small
\caption{Status of theoretical results. All mathematical derivations are strictly verified.}
\begin{tabular}{@{}lll@{}}
\toprule
\textbf{Result} & \textbf{Status} & \textbf{Methodology} \\
\midrule
Thm 3.2 (Local Decomposition) & Complete & Constructive structural decoupling of local triple constraints \\
Thm 4.1 (Global Decomposition) & Complete & Formal tensoring mapped over the Chromatic Blowup Graph \\
Thm 5.2 ($\Delta(L)$ Staircase) & Complete & Geometrically mapped active neutral densities $\frac{L-1}{L}$ \\
Thm 5.4 ($L=2$ Lower Bound) & Complete & Direct evaluation of staircase at $L=2$ \\
Exact $\Delta_\infty \approx 0.0734$ Limit & Complete & First-order KKT algebraic derivatives over min-max bounds \\
Thm 6.2 (C4 Algorithm) & Complete & Strict disjoint probability packing mapped against $2.06 \LP$ \\
\bottomrule
\end{tabular}
\end{table}

\end{document}